\documentclass[twoside]{article}
\usepackage[accepted]{aistats2024}

\usepackage[round]{natbib}

\usepackage{algorithm}
\usepackage{algorithmicx}
\usepackage[noend]{algpseudocode}
\usepackage{amsmath}
\usepackage{amssymb}
\usepackage{amsthm}
\usepackage[bibliography=common]{apxproof}
\usepackage{bbm}
\usepackage{bm}
\usepackage{caption}
\usepackage{color}
\usepackage{dirtytalk}
\usepackage{dsfont}
\usepackage{enumerate}
\usepackage{graphicx}
\usepackage{listings}
\usepackage{mathtools}
\usepackage{subfigure}
\usepackage{xspace}
\usepackage{soul}

\usepackage[usenames,dvipsnames]{xcolor}
\usepackage[bookmarks=false]{hyperref}
\hypersetup{
  pdffitwindow=true,
  pdfstartview={FitH},
  pdfnewwindow=true,
  colorlinks,
  linktocpage=true,
  linkcolor=Green,
  urlcolor=Green,
  citecolor=Green
}
\usepackage[capitalize,noabbrev]{cleveref}

\usepackage[textsize=tiny,textwidth=0.75in]{todonotes}

\newtheoremrep{definition}{Definition}
\newtheoremrep{lemma}{Lemma}
\newtheoremrep{theorem}{Theorem}

\newcommand{\cA}{\mathcal{A}}

\newcommand{\cD}{\mathcal{D}}

\newcommand{\cS}{\mathcal{S}}

\newcommand{\cX}{\mathcal{X}}

\newcommand{\realset}{\mathbb{R}}

\newcommand{\vol}{\mathrm{vol}}

\newcommand{\condE}[2]{\mathbb{E} \left[#1 \,\middle|\, #2\right]}
\newcommand{\Erv}[2]{\mathbb{E}_{#1} \left[#2\right]}

\newcommand{\abs}[1]{\left|#1\right|}

\newcommand*\dif{\mathop{}\!\mathrm{d}}

\newcommand{\I}[1]{\mathds{1} \! \left\{#1\right\}}

\newcommand{\set}[1]{\left\{#1\right\}}

\newcommand{\T}{^\top}

\DeclareMathOperator*{\argmax}{arg\,max\,}

\mathchardef\mhyphen="2D

\newcommand{\ucb}{\ensuremath{\tt UCB1}\xspace}

\newcommand{\ehvi}{\ensuremath{\tt EHVI}\xspace}

\newcommand{\meanhvi}{\ensuremath{\tt meanHVI}\xspace}
\newcommand{\oracle}{\ensuremath{\tt Oracle}\xspace}
\newcommand{\pesshvi}{\ensuremath{\tt pessHVI}\xspace}
\newcommand{\random}{\ensuremath{\tt Random}\xspace}

\begin{document}

\twocolumn[

\aistatstitle{Pessimistic Off-Policy Multi-Objective Optimization}

\aistatsauthor{Shima Alizadeh \And Aniruddha Bhargava \And Karthick Gopalswamy}

\aistatsaddress{AWS AI Labs \And Amazon \And AWS AI Labs}

\aistatsauthor{Lalit Jain \And Branislav Kveton \And Ge Liu}

\aistatsaddress{Amazon Visiting Scholar \And AWS AI Labs \And AWS AI Labs}]

\begin{abstract}
Multi-objective optimization is a type of decision making problems where multiple conflicting objectives are optimized. We study offline optimization of multi-objective policies from data collected by an existing policy. We propose a pessimistic estimator for the multi-objective policy values that can be easily plugged into existing formulas for hypervolume computation and optimized. The estimator is based on inverse propensity scores (IPS), and improves upon a naive IPS estimator in both theory and experiments. Our analysis is general, and applies beyond our IPS estimators and methods for optimizing them. The pessimistic estimator can be optimized by policy gradients and performs well in all of our experiments.
\end{abstract}

\section{Introduction}
\label{sec:introduction}

\emph{Multi-objective optimization (MOO)} is an area of decision making where multiple conflicting objectives are optimized \citep{keeney93decisions,emmerich18tutorial}. Many real-world problems have multiple objectives, including in economics \citep{ponsich13survey}, engineering \citep{marler04finding}, product design and manufacturing \citep{wang11multiobjective}, and logistics \citep{xifeng13multiobjective}. Therefore, MOO has been applied widely and successfully. A typical setting is of a system designer that tries to trade off multiple objectives subject to their preferences. As an example, when designing a product, the form factor, cost, and failure rate must be carefully balanced.

MOO has been usually studied under the assumption that the objective function is known, with a focus on optimizing it. When it is not known, the problem of learning to optimize it online can be formulated as a \emph{contextual bandit} \citep{li10contextual,chu11contextual}, where the goal is to learn a \emph{policy} that takes the most rewarding \emph{action} in each \emph{context}. In many applications, policies cannot be learned online by bandit algorithms because exploration can significantly impact user experience. However, offline data collected by a previously deployed policy are often available. Offline, or \emph{off-policy}, optimization using such logged data is a practical way of learning policies without costly online interactions \citep{dudik14doubly,swaminathan15counterfactual}. In this work, we study offline optimization of multi-objective policies from logged data.

One motivating example for our work is the design of a movie recommendation policy at a movie streaming company. As a first step, a policy could be learned offline to maximize the click-through rate (CTR). After it is deployed online, the policy recommends too many recent movies, which was not intended. To avoid this bias, a recent movie penalty is added to the objective and a new policy is learned offline. After it is deployed online, the policy recommends mostly classic movies, which was again not intended. Hence it needs to be redesigned again. A company typically goes through many iterations like this until a policy with a good balance between recency, popularity, and relevance is learned. We propose an offline framework for policy optimization that avoids this.

We study off-policy MOO from logged data and make the following contributions. First, we formalize offline optimization of multi-objective policies as hypervolume maximization. Second, we propose a pessimistic IPS estimator for the values of multi-objective policies that can be easily plugged into existing formulas for hypervolume computation. Third, we analyze the error of the estimator when used in optimization, and show its benefits over a naive IPS estimator. Our analysis is general, and applies beyond our IPS estimators (\cref{sec:off-policy estimation}) and methods for optimizing them (\cref{sec:hypervolume optimization}). Finally, we show the benefit of pessimistic optimization empirically on all major multi-objective benchmarks: ZDT \citep{zitzler2000comparison}, DTLZ \citep{deb2005scalable}, and WFG \citep{huband05scalable}.

\section{Setting}
\label{sec:setting}

\begin{figure}[t]
  \centering
  \includegraphics[trim={1.9in 0.55in 1.9in 0.55in},clip,width=3in]{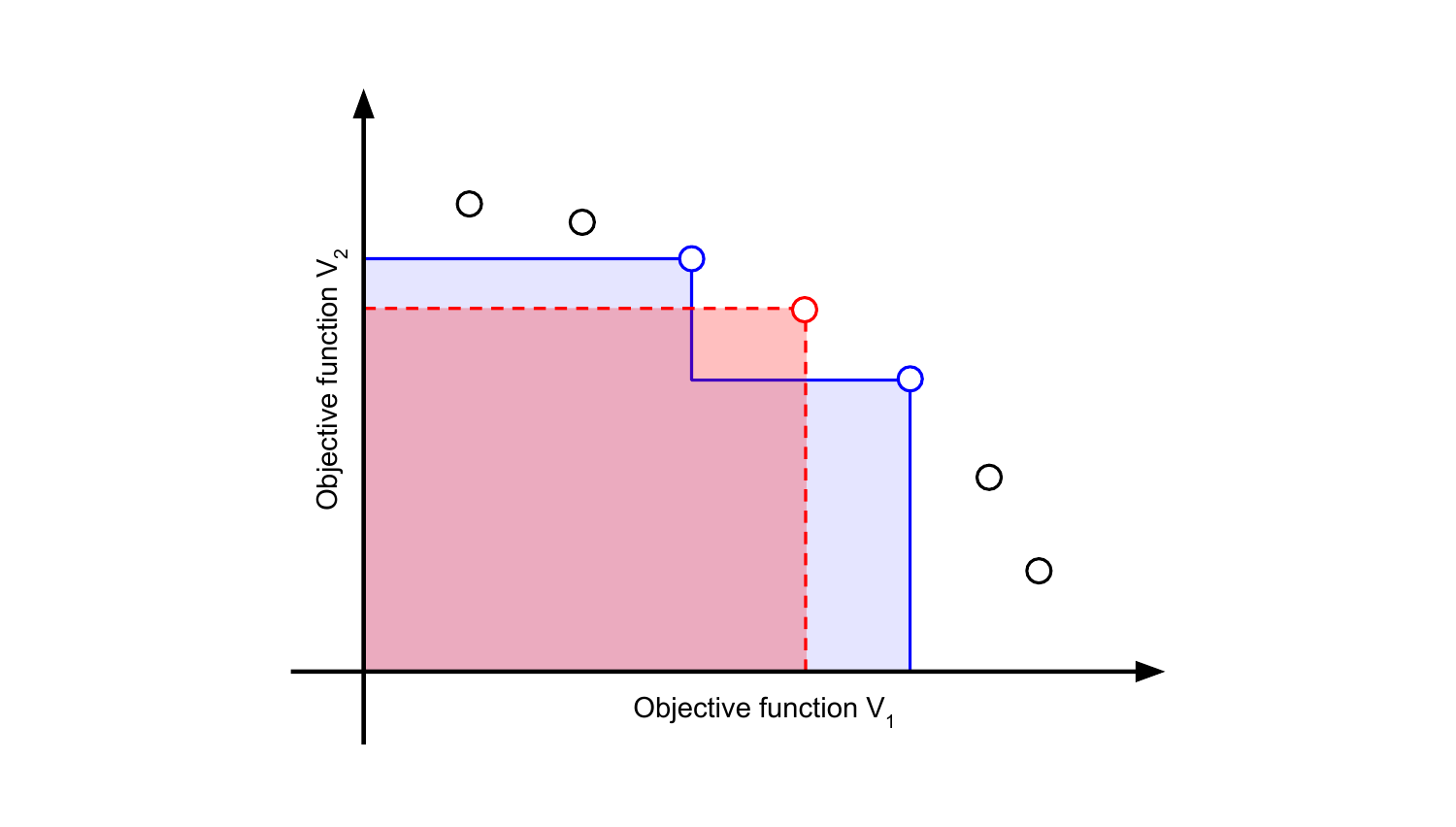}
  \caption{Each point is a value function $V(\pi)$ for one policy $\pi \in \Pi$. The red rectangle is the optimal hypervolume for $K = 1$. The union of the blue rectangles is the optimal hypervolume for $K = 2$.}
  \label{fig:illustration}
\end{figure}

We formally introduce the problem of policy optimization with a single objective in \cref{sec:single objective} and generalize it to multiple objectives in \cref{sec:multiple objectives}.

\subsection{Single-Objective Policy Optimization}
\label{sec:single objective}

We start with introducing our notation. Random variables are capitalized, except for Greek letters like $\theta$. For any positive integer $n$, we define $[n] = \set{1, \dots, n}$. The indicator function is $\I{\cdot}$. The $i$-th entry of vector $v$ is denoted by $v_i$. If the vector is already indexed, such as $v_j$, we write $v_{j, i}$.

In the classic contextual bandit \citep{li10contextual}, the agent observes a \emph{context} $x \in \cX$, where $\cX$ is a \emph{context set}; takes an \emph{action} $a \in \cA$, where $\cA$ is an \emph{action set}; and observes a \emph{stochastic reward} $Y \sim P(\cdot \mid x, a)$, where $P(\cdot \mid x, a)$ is the \emph{reward distribution} of action $a$ in context $x$. We denote the mean reward of action $a$ in context $x$ by $r(x, a) = \Erv{Y \sim P(\cdot \mid x, a)}{Y}$. A \emph{policy} $\pi$ maps actions to contexts, and we denote by $\pi(a \mid x)$ the probability of taking action $a$ in context $x$.

Let $(x_t)_{t = 1}^n$ be a sequence of $n$ contexts. The \emph{expected value of policy} $\pi$ in contexts $(x_t)_{t = 1}^n$ is
\begin{align}
  V(\pi)
  = \frac{1}{n} \sum_{t = 1}^n \sum_{a \in \cA} \pi(a \mid x_t) r(x_t, a)\,.
  \label{eq:policy value}
\end{align}
The optimal policy maximizes the expected value,
\begin{align}
  \textstyle
  \pi_*
  = \argmax_{\pi \in \Pi} V(\pi)\,,
  \label{eq:optimal policy}
\end{align}
where $\Pi$ is a class of optimized policies. If the policy class is sufficiently expressive, $\pi_*$ could take the action with the highest mean reward in each context.

\subsection{Multi-Objective Policy Optimization}
\label{sec:multiple objectives}

Now we extend the setting in \cref{sec:single objective} to multiple objectives. The main difference is that the stochastic reward $Y \sim P(\cdot \mid x, a)$ and its mean $r(x, a)$ are vectors of length $m$, where $m$ is the \emph{number of objectives}. We denote by $Y_i$ and $r_i(x, a)$ the corresponding rewards in objective $i \in [m]$. The expected value of policy $\pi$, $V(\pi)$ in \eqref{eq:policy value}, is also an $m$-dimensional vector; and we denote by $V_i(\pi)$ the value in objective $i$. We refer to $V(\pi)$ as a \emph{value function} since it maps policies to their values in multiple objectives. To simplify exposition, we assume that the stochastic rewards are bounded in $[0, 1]^m$. Thus $r(x, a) \in [0, 1]^m$ and $V(\pi) \in [0, 1]^m$.

Our motivating movie recommendation problem can be formulated in our setting as follows. The context set $\cX$ is the set of all users and the action set $\cA$ is the set of all movies that can be recommended. The user in interaction $t$, $x_t \in \cX$, is recommended movie $a \in \cA$ with probability $\pi(a \mid x_t)$. The mean reward could have two components, the click probability $r_1(x_t, a)$ and the purchase probability $r_2(x_t, a)$.

The main challenge in extending the optimization in \eqref{eq:optimal policy} to multiple objectives is that no policy may dominate others in all objectives. To address this problem, we adopt a standard approach in \emph{a-posteriori} MOO \citep{miettinen98nonlinear}: we cover the Pareto front of $V$ with diverse solutions $\pi \in \Pi$. Popular approaches for diversity maximization include random scalarization \citep{murata95moga}, Pareto dominance \citep{deb02fast}, and hypervolume maximization \citep{emmerich05emo}. We adopt the last approach. In our movie recommendation problem, the diverse set of policies would be learned and presented to a human decision maker, which would then select a policy that best fits their preferences.

We measure the diversity of policies by their \textit{hypervolume indicator}, a popular metric in multi-objective optimization \citep{emmerich05emo}. The \emph{hypervolume indicator} of policies $S \subseteq \Pi$ is defined as
\begin{align}
  \vol(S, V)
  & = \int_{y \in [0, 1]^m} \I{\bigvee_{\pi \in S} \{y \leq V(\pi)\}} \dif y
  \label{eq:hypervolume} \\
  & = \bigcup_{\pi \in S} \bigtimes_{i = 1}^m [0, V_i(\pi)]\,,
  \nonumber
\end{align}
where the inequality $y \leq V(\pi)$ is applied entry-wise. The first definition says that it is the fraction of points $y \in [0, 1]^m$ such that $y \leq V(\pi)$ holds for at least one $\pi \in S$. The second definition says that it is the hypervolume of a union of hyperrectangles corresponding to policies $\pi \in S$. To simplify terminology, we refer to \eqref{eq:hypervolume} as the \emph{hypervolume}.

Our goal is to identify $\hat{S} \subseteq \Pi$ such that $|\hat{S}| \leq K$ and $\vol(\hat{S}, V) \approx \vol(\Pi, V)$. Roughly speaking, $\hat{S}$ should be as diverse as $\Pi$, as measured by covering a similar space. Thus a natural generalization of \eqref{eq:optimal policy} is the set of $K$ policies that maximizes the hypervolume,
\begin{align}
  \textstyle
  S_*
  = \argmax_{S \subseteq \Pi: |S| = K} \vol(S, V)\,.
  \label{eq:optimal policy set}
\end{align}
We note that \eqref{eq:optimal policy set} reduces to \eqref{eq:optimal policy} when the number of objectives is $m = 1$. We illustrate solutions to \eqref{eq:optimal policy set} for $K \in \set{1, 2}$ in \cref{fig:illustration}.

In this work, we study a setting where the value function $V$, an input to $\vol(S, V)$, is estimated from logged data. We give estimators of $V$ in \cref{sec:off-policy estimation} and then analyze them in \cref{sec:analysis}. We discuss algorithms for solving \eqref{eq:optimal policy set} using the estimators in \cref{sec:hypervolume optimization}.

\section{Off-Policy Multi-Objective Estimation}
\label{sec:off-policy estimation}

The value function $V$ is unknown but we can estimate it from data collected by another policy. We formalize this problem as follows. Let $(x_t)_{t = 1}^n$ be the same sequence of contexts as in \eqref{eq:policy value} and $\pi_0$ be a data \emph{logging policy} $\pi_0$, which takes action $A_t \sim \pi_0(\cdot \mid x_t)$ in interaction $t \in [n]$. Let $Y_t = (Y_{t, i})_{i \in [m]}$ be the corresponding reward, generated as $Y_t \sim P(\cdot \mid x_t, A_t)$. The rewards are \emph{stochastic} and sampled independently, with means $\condE{Y_{t, i}}{x_t, A_t} = r_i(x_t, A_t)$ and $\sigma^2$-sub-Gaussian noise. The result of this process is a \emph{logged dataset} of size $n$, $\cD = \set{(x_t, A_t, Y_t)}_{t \in [n]}$, which we use to estimate $V$.

The rest of this section is organized as follows. We present an inverse propensity score estimator of the value function $V$ in \cref{sec:ips}. Our main contribution is its pessimistic variant in \cref{sec:pessimistic ips}. We focus on these estimators because they can be easily combined with differentiable policies \citep{swaminathan15counterfactual}. We discuss other possible choices in \cref{sec:other estimators}. We design estimators separately for each objective, which allows them to be plugged into existing hypervolume estimators. For instance, if $\hat{V}_i(\pi)$ is an estimate of $V_i(\pi)$, we only need to replace $V_i(\pi)$ in \eqref{eq:hypervolume} to obtain hypervolume under that estimate. While the per-objective design appeared before in \citet{wang22imo3}, we are the first to incorporate pessimism and confidence intervals.

\subsection{IPS Estimator}
\label{sec:ips}

One approach to off-policy optimization is to optimize the mean estimate $\hat{V}$ of $V$. The advantage is that the uncertainty of $\hat{V}$ does not have to be modeled. The most popular approach for estimating the mean value of a policy are \emph{inverse propensity scores (IPS)} \citep{horwitz52generalization}. Generally, the value of policy $\pi$ in objective $i$ can be estimated using a \emph{clipped IPS estimator} \citep{ionides08truncated,strehl10learning} as
\begin{align}
  \hat{V}_i(\pi, M)
  = \frac{1}{n} \sum_{t = 1}^n
  \min \set{\frac{\pi(A_t \mid x_t)}{\pi_0(A_t \mid x_t)}, M} Y_{t, i}\,,
  \label{eq:ips}
\end{align}
where $M \geq 0$ is a clipping parameter to control bias. The higher the value of $M$, the lower the bias and the higher the variance. When $M = \infty$, the estimator is unbiased but may suffer from a high variance. When $M = 0$, the estimator returns $0$ for any policy $\pi$. We define the \emph{IPS estimator} as $\hat{V}_i(\pi) = \hat{V}_i(\pi, \infty)$.

\subsection{Pessimistic IPS Estimator}
\label{sec:pessimistic ips}

Another approach to off-policy optimization is based on pessimism \citep{swaminathan15counterfactual,jin21is,hong23multitask}, where a \emph{lower confidence bound (LCB)} is optimized. The LCB can be derived using a high-probability confidence interval, which we derive below.

\begin{lemma}
\label{lem:confidence interval} Let $c_i(\pi) = \beta \sigma M_\pi / n$ for $\beta > 0$ and
\begin{align}
  M_\pi^2
  = \sum_{t = 1}^n M_{t, \pi}^2\,, \quad
  M_{t, \pi}
  = \max_{a \in \cA} \frac{\pi(a \mid x_t)}{\pi_0(a \mid x_t)}\,.
  \label{eq:maximum propensity score}
\end{align}
Then for any objective $i \in [m]$ and policy $\pi \in \Pi$, the bound $|\hat{V}_i(\pi) - V_i(\pi)| \leq c_i(\pi)$ holds with probability at least $1 - 2 \exp[- \beta^2 / 2]$.
\end{lemma}
\begin{proof}
First, note that $\hat{V}_i(\pi)$ is a weighted sum of independent $\sigma^2$-sub-Gaussian rewards $Y_{t, i}$ and its mean is $V_i(\pi)$. Second, each reward $Y_{t, i}$ is scaled by at most $M_{t, \pi}$. Therefore, $\hat{V}_i(\pi)$ is sub-Gaussian with variance proxy $\sigma^2 M_\pi^2 / n^2$; and the stated claim follows from standard concentration bounds for sub-Gaussian random variables \citep{boucheron13concentration}.
\end{proof}

The corresponding lower confidence bound is
\begin{align}
  L_i(\pi)
  = \hat{V}_i(\pi) - c_i(\pi)
  \label{eq:pessimistic ips}
\end{align}
and we call it a \emph{pessimistic IPS estimator}. When $\beta = \sqrt{2 \log(2 / \delta)}$, the LCB holds with probability at least $1 - \delta$ for any objective $i \in [m]$ and policy $\pi \in \Pi$. We note that $L_i(\pi)$ can be negative. Therefore, it should be clipped from below by $0$ before plugging it into \eqref{eq:hypervolume} instead of $V_i(\pi)$.

One notable property of $c_i(\pi)$ is that it captures similarities of policies $\pi$ and $\pi_0$. More specifically, $c_i(\pi) = O(M_\pi)$, where $M_\pi$ in \eqref{eq:maximum propensity score} is a sum of maximum ratios between the probabilities of taking actions under policies $\pi$ and $\pi_0$. Therefore, $c_i(\pi)$ decreases as $\pi \to \pi_0$ and so does the uncertainty in the estimate of $V_i(\pi)$.

\section{Analysis}
\label{sec:analysis}

\begin{figure}[t]
  \centering
  \includegraphics[trim={1.9in 0.55in 1.9in 0.55in},clip,width=3in]{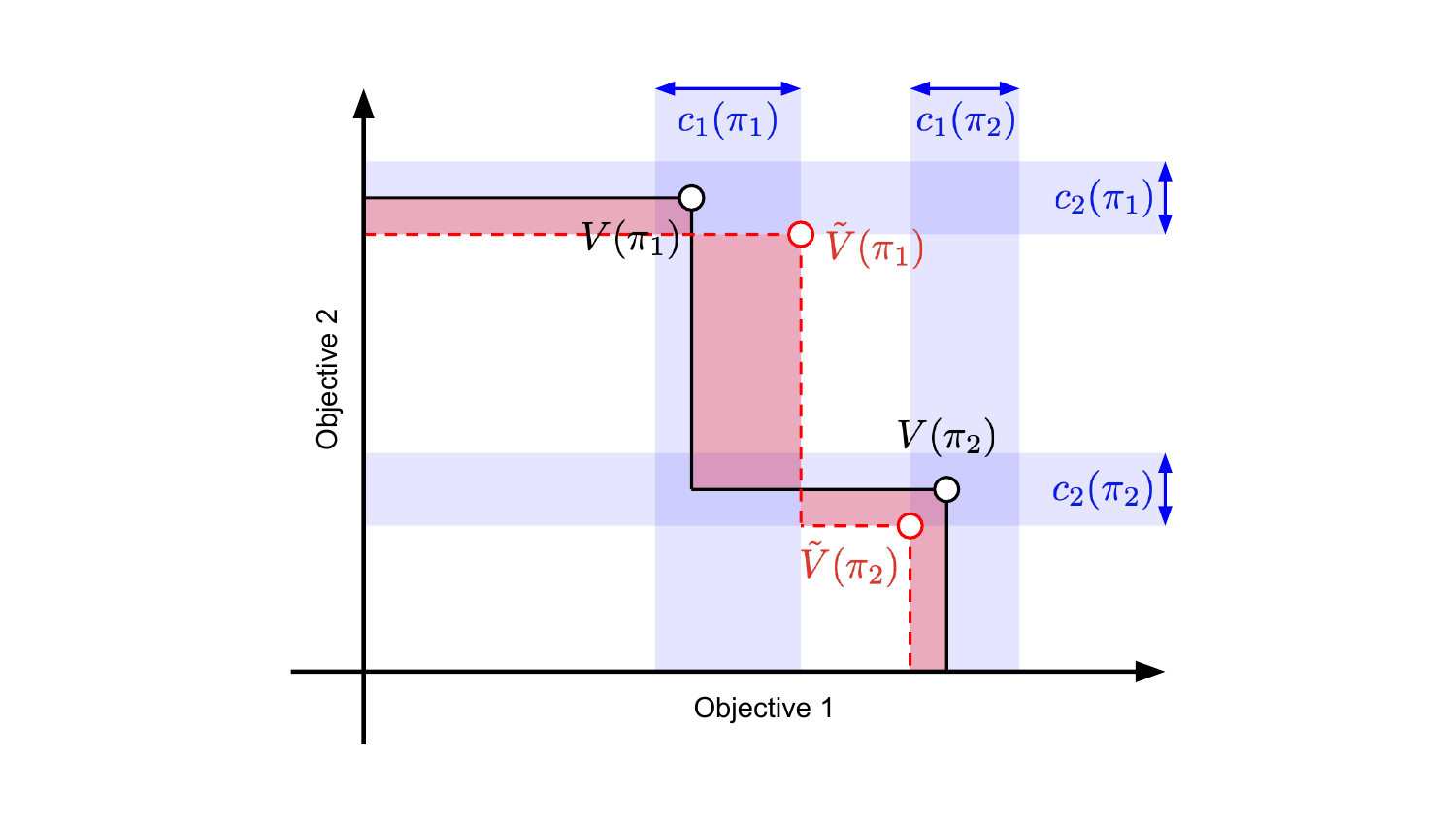}
  \caption{Illustration of \cref{lem:hypervolume approximation error} for $S = \set{\pi_1, \pi_2}$. The blue stripes represent $\sum_{\pi \in S} \sum_{i = 1}^2 c_i(\pi)$ and their areas bound the hypervolume difference (red area).}
  \label{fig:error bound}
\end{figure}

Next we analyze the benefit of acting pessimistically. Our analysis assumes access to $\alpha$-approximate maximization oracles.

\begin{definition}
An oracle is $\alpha$-approximate if for any value function $\tilde{V}: \Pi \to [0, 1]^m$, it computes a set of $K$ policies $\tilde{S} = {\oracle}(\vol(\cdot, \tilde{V}))$ such that
\begin{align*}
  \textstyle
  \vol(\tilde{S}, \tilde{V})
  \geq \alpha \max_{S \subseteq \Pi: |S| = K} \vol(S, \tilde{V})\,.
\end{align*}
\end{definition}

This assumption allows us to study the statistical efficiency of our estimators without being worried about the computational cost of maximizing them. As shown later (\cref{sec:discussion}), the quality of the oracle affects all our bounds identically. Thus the benefit of pessimism can be argued for any oracle and more abstract treatment is appropriate. Note that when $\Pi$ is discrete and small, a computationally-efficient maximization oracle exists for $\alpha = 1 - 1 / e$ (\cref{sec:discrete optimization}).

This section is organized as follows. In \cref{sec:approximate maximization}, we derive error bounds for approximate maximization of functions using their mean and pessimistic estimates. In \cref{sec:hypervolume maximization}, we specialize the bounds to approximate hypervolume maximization. Finally, we compare the bounds in \cref{sec:discussion}. The main novelty in our analysis is that we decompose the uncertainty of hypervolume into those of its points, and show its effect on optimization. Our bounds are general, and apply beyond our IPS estimators in \cref{sec:off-policy estimation} and methods for optimizing them in \cref{sec:hypervolume optimization}. All omitted proofs are in \cref{proof:analysis}.

\subsection{Approximate Maximization}
\label{sec:approximate maximization}

We find it convenient to use a more abstract approach generalizing to arbitrary set functions. Let $\Pi$ be a set of points and $\cS \subseteq 2^\Pi$ be a subset of its power set. Let $g: \cS \to \realset$ be a set function with an approximation $\tilde{g}: \cS \to \realset$. For any set $S \in \cS$, let $c(S) \geq |g(S) - \tilde{g}(S)|$ be an upper bound on the approximation error. Let $S_* = \argmax_{S \in \cS} g(S)$ be the maximizer of $g$. Then the $\alpha$-approximate maximum of $\tilde{g}$ and the true maximum of $g$ can be related as follows.

\begin{lemmarep}
\label{lem:mean} Let $\tilde{S} = \argmax_{S \in \cS} \tilde{g}(S)$ maximize $\tilde{g}$ and $\hat{S}$ be its $\alpha$-approximation, so that $\tilde{g}(\hat{S}) \geq \alpha \tilde{g}(\tilde{S})$ holds for some $\alpha \in [0, 1]$. Then
\begin{align*}
  g(\hat{S})
  \geq \alpha g(S_*) - c(S_*) - c(\hat{S})\,.
\end{align*}
\end{lemmarep}
\begin{appendixproof}
\label{proof:analysis} The claim is proved as
\begin{align*}
  \alpha g(S_*) - g(\hat{S})
  & = \alpha g(S_*) - \alpha \tilde{g}(S_*) + \alpha \tilde{g}(S_*) - g(\hat{S}) \\
  & \leq \alpha (g(S_*) - \tilde{g}(S_*)) + \alpha \tilde{g}(\tilde{S}) - g(\hat{S}) \\
  & \leq \alpha (g(S_*) - \tilde{g}(S_*)) + \tilde{g}(\hat{S}) - g(\hat{S}) \\
  & \leq c(S_*) + c(\hat{S})\,.
\end{align*}
The first inequality holds because $\tilde{S}$ maximizes $\tilde{g}$. The second inequality uses that $\hat{S}$ is an $\alpha$-approximation. The last inequality follows from the definition of function $c$ and $\alpha \in [0, 1]$.
\end{appendixproof}

The following bound holds for maximizing a lower bound $L(S) = \tilde{g}(S) - c(S)$ on $g(S)$.

\begin{lemmarep}
\label{lem:pessimism} Let $\tilde{S} = \argmax_{S \in \cS} L(S)$ maximize $L$ and $\hat{S}$ be its $\alpha$-approximation, so that $L(\hat{S}) \geq \alpha L(\tilde{S})$ holds for some $\alpha \in [0, 1]$. Then
\begin{align*}
  g(\hat{S})
  \geq \alpha g(S_*) - 2 c(S_*)\,.
\end{align*}
\end{lemmarep}
\begin{appendixproof}
The claim is proved as
\begin{align*}
  \alpha g(S_*) - g(\hat{S})
  & = \alpha g(S_*) - \alpha L(S_*) + \alpha L(S_*) - g(\hat{S}) \\
  & \leq \alpha (g(S_*) - L(S_*)) + \alpha L(\tilde{S}) - g(\hat{S}) \\
  & \leq \alpha (g(S_*) - L(S_*)) + L(\hat{S}) - g(\hat{S}) \\
  & \leq 2 c(S_*)\,.
\end{align*}
The first inequality holds because $\tilde{S}$ maximizes $L$ and the second inequality uses that $\hat{S}$ is an $\alpha$-approximation. The last inequality follows from $L(S_*) = \hat{g}(S_*) - c(S_*)$ and $L(\hat{S}) - g(\hat{S}) \leq 0$. After that, we use the definition of function $c$ and $\alpha \in [0, 1]$.
\end{appendixproof}

Optimization of the estimated mean (\cref{lem:mean}) and the lower bound (\cref{lem:pessimism}) differ as follows. In the former, the error can be arbitrarily large if $\tilde{g}$ significantly overestimates $g$ on some set $S$. In the latter, the error is bounded by the error at the optimal solution $S_*$ only. Arguably, if this one is high, $S_*$ is hard to identify. Therefore, maximization of a lower bound yields more robust solutions.

\subsection{Hypervolume Maximization}
\label{sec:hypervolume maximization}

Now we use the error bounds in \cref{sec:approximate maximization} to analyze IPS hypervolume maximization (\cref{sec:ips}). The key step in our argument is to relate the hypervolume under an approximation $\tilde{V}$ to that under the true function $V$. We relate these through errors in individual objective estimates.

\begin{lemmarep}
\label{lem:hypervolume approximation error} Let $V_i(\pi), \tilde{V}_i(\pi) \in [0, 1]$ for all $i \in [m]$ and $\pi \in \Pi$. Assume that $|V_i(\pi) - \tilde{V}_i(\pi)| \leq c_i(\pi)$ for all $i \in [m]$ and $\pi \in \Pi$. Then
\begin{align*}
  |\vol(S, V) - \vol(S, \tilde{V})|
  \leq c(S)
  = \sum_{\pi \in S} \sum_{i = 1}^m c_i(\pi)\,.
\end{align*}
\end{lemmarep}
\begin{appendixproof}
We start with the observation that for any two vectors $a, b \in \set{0, 1}^d$,
\begin{align}
  \abs{\prod_{i = 1}^d a_i - \prod_{i = 1}^d b_i}
  \leq \sum_{i = 1}^d \abs{a_i - b_i}\,, \quad
  \abs{1 - \prod_{i = 1}^d (1 - a_i) -
  \left(1 - \prod_{i = 1}^d (1 - b_i)\right)}
  \leq \sum_{i = 1}^d \abs{a_i - b_i}\,.
  \label{eq:and or}
\end{align}
In plain English, the difference in the logical \say{and} and \say{or} over entries of these vectors is bounded by the sum of the differences of their entries. The definition of the hypervolume together with these inequalities yields
\begin{align*}
  |\vol(S, V) - \vol(S, \tilde{V})|
  & \leq \int_{y \in [0, 1]^m} \abs{\I{\bigvee_{\pi \in S} \{y \leq V(\pi)\}} -
  \I{\bigvee_{\pi \in S} \{y \leq \tilde{V}(\pi)\}}} \dif y \\
  & \leq \sum_{\pi \in S} \int_{y \in [0, 1]^m}
  \abs{\I{y \leq V(\pi)} - \I{y \leq \tilde{V}(\pi)}} \dif y \\
  & \leq \sum_{\pi \in S} \sum_{i = 1}^m \int_{y \in [0, 1]}
  \abs{\I{y \leq V_i(\pi)} - \I{y \leq \tilde{V}_i(\pi)}} \dif y \\
  & = \sum_{\pi \in S} \sum_{i = 1}^m \abs{V_i(\pi)- \tilde{V}_i(\pi)}
  \leq \sum_{\pi \in S} \sum_{i = 1}^m c_i(\pi)
  = c(S)\,.
\end{align*}
In the first and second inequalities, we use the \say{or} and \say{and} inequalities in \eqref{eq:and or}, respectively. The rest follows from basic integration identities and that we integrate over a $[0, 1]^m$ hypercube.
\end{appendixproof}

The lemma says that the difference in the hypervolume of $S$ under $V$ and $\tilde{V}$ is bounded by the sum of differences of $V_i$ and $\tilde{V}_i$ in individual policies $\pi \in S$. We visualize this in \cref{fig:error bound}.

To obtain an error bound for IPS hypervolume maximization, we chain \cref{lem:mean,lem:hypervolume approximation error}. Our analysis is under the assumption that \eqref{eq:ips} is clipped to $[0, 1]$.

\begin{theorem}
\label{thm:ips} Let $V_i(\pi), \hat{V}_i(\pi) \in [0, 1]$ for all $i \in [m]$ and $\pi \in \Pi$. Let $|V_i(\pi) - \hat{V}_i(\pi)| \leq c_i(\pi)$ hold jointly for all $i \in [m]$ and $\pi \in \Pi$ with probability at least $1 - \delta$. Let $\hat{S} = {\oracle}(\vol(\cdot, \hat{V}))$ be an $\alpha$-approximate solution. Then with probability at least $1 - \delta$,
\begin{align*}
  \vol(\hat{S}, V)
  \geq \alpha \vol(S_*, V) - c(S_*) - c(\hat{S})\,,
\end{align*}
where $c(S) = \sum_{\pi \in S} \sum_{i = 1}^m c_i(\pi)$.
\end{theorem}
\begin{proof}
Since $|V_i(\pi) - \hat{V}_i(\pi)| \leq c_i(\pi)$, by \cref{lem:hypervolume approximation error}
\begin{align*}
  |\vol(S, V) - \vol(S, \hat{V})|
  \leq c(S)\,.
\end{align*}
Now we apply \cref{lem:mean}, where $g(S) = \vol(S, V)$ and $\tilde{g}(S) = \vol(S, \hat{V})$.
\end{proof}

\cref{thm:ips} says that $\vol(\hat{S}, V)$ is within a multiplicative factor of $\alpha$ of $\vol(S_*, V)$. The additional error depends on the magnitude of confidence intervals $c_i(\pi)$ for $\pi \in S_* \cup \hat{S}$. We discuss this more in \cref{sec:discussion}.

Now we use the error bounds from \cref{sec:approximate maximization} to analyze pessimistic IPS hypervolume maximization (\cref{sec:pessimistic ips}). The proof is analogous to \cref{thm:ips}, with the only difference that we apply \cref{lem:pessimism} instead of \cref{lem:mean}. We assume that \eqref{eq:pessimistic ips} is clipped to $[0, 1]$.

\begin{theorem}
\label{thm:pessimistic ips} Let $L_i(\pi) = \hat{V}_i(\pi) - c_i(\pi)$ be a lower confidence bound. Let $V_i(\pi), L_i(\pi) \in [0, 1]$ for all $i \in [m]$ and $\pi \in \Pi$. Let $|V_i(\pi) - \hat{V}_i(\pi)| \leq c_i(\pi)$ hold jointly for all $i \in [m]$ and $\pi \in \Pi$ with probability at least $1 - \delta$. Let $\hat{S} = {\oracle}(\vol(\cdot, L))$ be an $\alpha$-approximate solution. Then with probability at least $1 - \delta$,
\begin{align*}
  \vol(\hat{S}, V)
  \geq \alpha \vol(S_*, V) - 2 c(S_*)\,,
\end{align*}
where $c(S) = \sum_{\pi \in S} \sum_{i = 1}^m c_i(\pi)$.
\end{theorem}
\begin{proof}
Since $|V_i(\pi) - \hat{V}_i(\pi)| \leq c_i(\pi)$, by \cref{lem:hypervolume approximation error}
\begin{align*}
  |\vol(S, V) - \vol(S, \hat{V})|
  \leq c(S)\,.
\end{align*}
Now we apply \cref{lem:pessimism}, where $g(S) = \vol(S, V)$ and $\tilde{g}(S) = \vol(S, L)$.
\end{proof}

\cref{thm:pessimistic ips} says that $\vol(\hat{S}, V)$ is within a multiplicative factor of $\alpha$ of $\vol(S_*, V)$. The additional error depends on the magnitude of confidence intervals $c_i(\pi)$ for $\pi \in S_*$. We discuss this more in \cref{sec:discussion}.

\subsection{Discussion}
\label{sec:discussion}

\cref{thm:ips,thm:pessimistic ips} are similar in two aspects. First, both say that the solution $\hat{S}$ is $\alpha$-approximate up to the uncertainty in the estimate of $V$. Second, the uncertainty is characterized by hypervolume confidence interval widths. Specifically, when $M_{t, \pi} = O(1)$ in \eqref{eq:maximum propensity score}, the widths are
\begin{align*}
  c(S)
  = O(\beta \sigma K m \sqrt{1 / n})\,.
\end{align*}
As expected, they increase as the reward noise $\sigma$ increases, the number of optimized policies $K$ increases, the number of objectives $m$ increases, and the probability $1 - \delta$ with which the guarantee holds increases. The decrease with a larger sample size $n$ is expected as well.

\cref{thm:ips,thm:pessimistic ips} differ only in confidence intervals. In \cref{thm:pessimistic ips}, the confidence interval depends on the optimal set of policies $S_*$ only. Therefore, when the logging policy $\pi_0$ is near optimal, $c(S_*)$ is small and pessimistic IPS maximization is comparable to maximizing $\vol(\cdot, V)$, even if $V$ is unknown and potentially poorly estimated everywhere else but around $\pi \in S_*$.

Such a guarantee cannot be proved for the IPS maximization in \cref{thm:ips}. To demonstrate this, suppose that there exists a policy $\hat{\pi} \in \Pi$ such that $\hat{V}_i(\hat{\pi}) \approx 1$ and $c_i(\hat{\pi}) \approx 1$ for all $i \in [m]$. Based on $\hat{V}_i(\hat{\pi})$ alone, $\hat{\pi}$ is a highly-rewarding policy. However, when the confidence intervals are considered, it is clear that $\hat{V}_i(\hat{\pi})$ are unreliable estimates. Therefore, it could be that $V_i(\hat{\pi}) \approx 0$ for all $i \in [m]$. This is captured by the term $c(\hat{S})$ in \cref{thm:ips}, which would be $O(1)$ when $\hat{\pi} \in \hat{S}$ and thus render the guarantee meaningless.

Finally, when the logging policy $\pi_0$ is uniform, we do not expect any benefit of pessimism because all confidence intervals, including $c(S_*)$ and $c(\hat{S})$, would have similar widths. \cref{thm:ips,thm:pessimistic ips} show this.

\section{Hypervolume Optimization}
\label{sec:hypervolume optimization}

Hypervolume optimization in \eqref{eq:optimal policy set} is a hard problem because both the maximization problem and hypervolume computation are. We borrow from prior works to address them. All discussions in this section apply to the value function $V$ in \eqref{eq:policy value}, its IPS estimator in \eqref{eq:ips}, and its pessimistic IPS estimator in \eqref{eq:pessimistic ips}.

\subsection{Discrete Optimization}
\label{sec:discrete optimization}

It is well known that $\vol(S, V)$ is monotone and submodular in $S$ \citep{ulrich2012bounding}. Therefore, it can be maximized greedily with guarantees as follows. In iteration $k \in [K]$, a policy $\pi_k \in \Pi$ that maximally increases the hypervolume, after being added to previously selected policies $\set{\pi_\ell}_{\ell = 1}^{k - 1}$, is chosen,
\begin{align}
  \pi_k
  = \argmax_{\pi \in \Pi} \vol(\set{\pi_1, \dots, \pi_{k - 1}, \pi}, V)\,.
  \label{eq:greedy maximization}
\end{align}
By \citet{nemhauser78approximation}, the greedy solution after $K$ iterations $\hat{S} = \set{\pi_k}_{k = 1}^K$ is $(1 - 1 / e)$-optimal, due to the monotonicity and submodularity of $\vol(S, V)$ in $S$.  Unfortunately, \eqref{eq:greedy maximization} requires $O(\abs{\Pi})$ hypervolume evaluations per iteration. Therefore, it is computationally inefficient when $\Pi$ is large and cannot be applied when $\Pi$ is continuous.

\subsection{Policy Gradient}
\label{sec:policy gradient}

Rather than being limited by discrete optimization in \cref{sec:discrete optimization}, we optimize a general policy class using policy gradients \citep{williams92simple,sutton00policy,baxter01infinitehorizon}. Let
\begin{align}
  \pi(a \mid x; \theta)
  = \frac{\exp[\phi(x, a)\T \theta]}{\sum_{a' \in \cA} \exp[\phi(x, a')\T \theta]}
  \label{eq:policy}
\end{align}
be the probability of taking action $a \in \cA$ in context $x \in \cX$ parameterized by \emph{policy parameter} $\theta \in \Theta$. Here $\phi: \cX \times \cA \to \realset^d$ is an arbitrary feature mapping and $\Theta \subseteq \realset^d$ is the space of policy parameters.

To solve \eqref{eq:optimal policy set}, we apply policy gradient to all $K$ optimized policies. Specifically, let $\theta_{\ell, k} \in \Theta$ be the parameter of policy $k$ in iteration $\ell$ and $\theta_\ell = \bigoplus_{k = 1}^K \theta_{\ell, k}$ be a concatenated parameter vector of all $K$ policies in iteration $\ell$. Then, in iteration $\ell$, we update $\theta_\ell$ as
\begin{align}
  \theta_{\ell + 1}
  = \theta_\ell + \alpha_\ell \nabla_{\theta_\ell}
  \vol(\set{\pi(\cdot \mid \cdot; \theta_{\ell, k})}_{k = 1}^K, V)\,,
  \label{eq:policy gradient}
\end{align}
where $\alpha_\ell$ is the learning rate in iteration $\ell$, which can be adapted. The hypervolume is differentiable in $\theta_\ell$ as long as $V_i(\pi(\cdot \mid \cdot; \theta_{\ell, k}))$ is differentiable in $\theta_{\ell, k}$. This is true for any policy of form \eqref{eq:policy} plugged into the value function in \eqref{eq:policy value}, its IPS estimator in \eqref{eq:ips}, or its pessimistic IPS estimator in \eqref{eq:pessimistic ips}.

In experiments, we implement \eqref{eq:policy gradient} by automatic differentiation with Adam \citep{kingma15adam}. This choice was driven by the popularity of Adam and its good initial performance. We discuss other potential choices in \cref{sec:diverse gradient}.

\subsection{Hypervolume Computation}
\label{sec:hypervolume computation}

Exact computation of the hypervolume of $K$ points is exponential in $K$, because it corresponds to computing the union of $K$ hyperrectangle volumes. Such computations are only feasible when $K$ is small (\cref{sec:inclusion-exclusion estimator}). Efficient exact algorithms also exist for $m = 2$ objectives (\cref{sec:two objectives}). In general, the computation can be reduced to Klee's measure problem and the best known computational complexity is $\tilde{O}(K^{\frac{m}{3}})$ \citep{chan2008slightly}. Despite this, many efficient approximation exists (\cref{sec:hypervolume}).

\section{Experiments}
\label{sec:experiments}

\begin{figure*}[t]
  \centering
  \includegraphics[width=17cm]{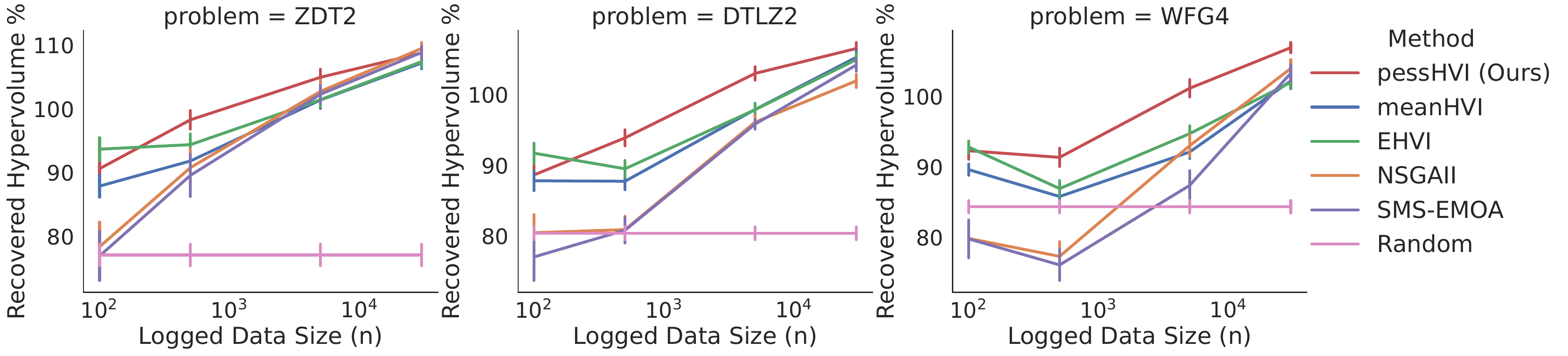}
  \vspace{-0.05in}
  \caption{Comparison of \pesshvi to baselines for $K = 10$ while varying logged dataset size $n$.}
  \label{fig:n}
\end{figure*}

\begin{figure*}[t]
  \centering
  \includegraphics[width=17cm]{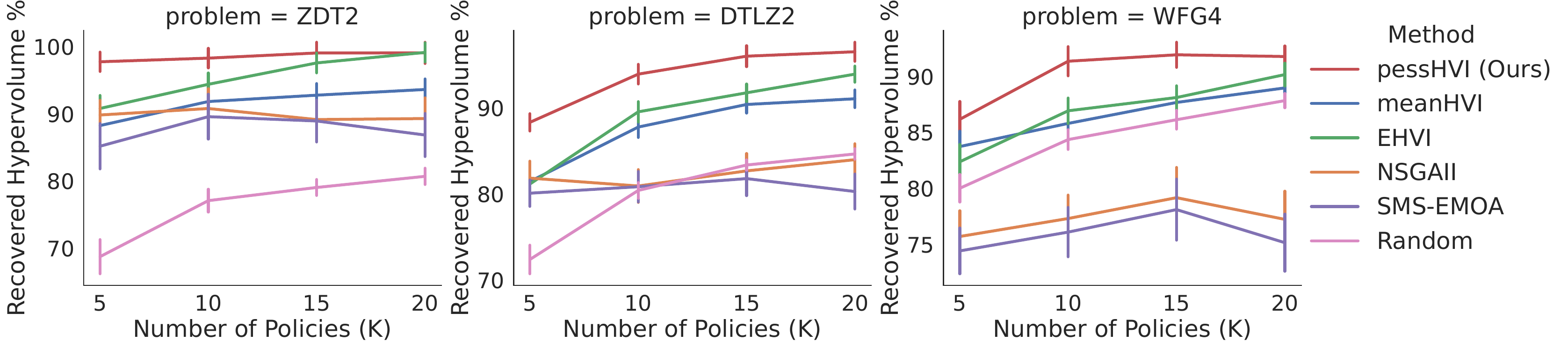}
  \vspace{-0.05in}
  \caption{Comparison of \pesshvi to baselines at $n = 500$ while varying the number of optimized policies $K$.}
  \label{fig:K}
\end{figure*}

We also evaluate the benefit of pessimism empirically. Due to space constraints, we only show representative trends and defer the rest to \cref{sec:additional experiments}.

\subsection{Benchmarks}
\label{sec:experiments benchmarks}

No standardized benchmarks exist for evaluating off-policy MOO. Therefore, we adapt three popular MOO bechmarks, which have been used in numerous works: ZDT \citep{zitzler2000comparison}, DTLZ \citep{deb2005scalable}, and WFG \citep{huband05scalable}. ZDT is a set of bi-objective problems where the number of features can vary. Both the number of objectives and features can vary in DTLZ and WFG problems.

At a high level, we use multi-objective functions in existing benchmarks to define the mean rewards $r(x, a)$ in \eqref{eq:policy value}. The mean rewards can be controlled through actions $a$ and we optimize policies over them. Specifically, let $\realset^d$ be the feature space of an existing benchmark and $f: \realset^d \to \realset^m$ be its multi-objective function. We split the feature space as $\realset^d = \cX \times \cA$, where $\cX = \realset^{\frac{d}{2}}$ and $\cA = \realset^{\frac{d}{2}}$ are the context and action sets, respectively. The mean reward for taking action $a \in \cA$ in context $x \in \cX$ is $r(x, a) = f(x \oplus a)$, where $u \oplus v$ is the concatenation of vectors $u$ and $v$. We discretize $\cA$ by $20$ random points to guarantee that the probability of taking an action in \eqref{eq:policy} can be properly normalized. The features are $\phi(x, a) = x \oplus a \oplus \mathrm{vec}(x a\T) \oplus (1)$, where $\mathrm{vec}(M)$ denotes the vectorization of matrix $M$. We introduce the cross-interaction term $x a\T$ to allow for context-dependent policies.

\subsection{Evaluation Protocol}
\label{sec:evalaution protocol}

\textbf{Compared methods.} Our method is a policy gradient (\cref{sec:policy gradient}) with the pessimistic IPS estimator in \eqref{eq:pessimistic ips}. We call it \pesshvi. We set $\beta$ in \cref{lem:confidence interval} to a practical value of $0.2$, which performed well in our initial experiments. To show the benefit of pessimism, we compare \pesshvi to a policy gradient with the IPS estimator in \eqref{eq:ips}. We call it \meanhvi.

We consider four additional baselines. The first baseline selects $K$ random policies of form \eqref{eq:policy}, where $\theta$ is randomly chosen from a unit ball. We call it \random. This baseline establishes what can be achieved with a minimal computational cost. The next two baselines are state-of-the-art genetic algorithms: NSGA-II \citep{deb02fast} and SMS-EMOA \citep{beume07smsemoa}. We implement them with the IPS estimator in \eqref{eq:ips}.

The last baseline is a state-of-the-art approach of \emph{expected hypervolume improvement (EHVI)} \citep{emmerich05emo,emmerich2005single,ernst05treebased,yang2019multi}. The main difficulty that we faced in implementing it was that our setting is not Bayesian, and thus there is no prior or posterior. At the end, we implement it using bootstrapping \citep{efron86bootstrap}, which is known to be equivalent to posterior sampling in several notable cases \citep{lu17ensemble,vaswani18new}. Specifically, we take the logged dataset $\cD$ and resample it $N$ times with replacement. Let $\tilde{V}_j(\pi)$ be the IPS estimate of value function $V(\pi)$ from resampled dataset $j \in [N]$. Treating it as a posterior sample, the expected hypervolume for policy $\pi$ could be approximated as $\frac{1}{N} \sum_{j = 1}^N \vol(S, \tilde{V}_j(\pi))$. We optimize EHVI using policy gradient over the policy class in \eqref{eq:policy}, exactly as in \pesshvi.

\textbf{Hypervolume computation.} All methods are described in \cref{sec:hypervolume}. We use the exact formula in \cref{sec:two objectives} for $m = 2$ objectives. For $m > 2$, we use the scalarized approximation in \cref{sec:random hypervolume scalarization}.

\textbf{Logging policy.} We define the policy as follows. For any context $x \in \cX$, let $F_{x, a}$ indicate that $r(x, a)$ is on the Pareto front of $\set{r(x, a)}_{a \in \cA}$. Then
\begin{align}
  \pi_0(a \mid x)
  \propto \varepsilon / \abs{\cA} +
  (1 - \varepsilon) F_{x, a} \Big/ \sum_{a \in \cA} F_{x, a}\,.
  \label{eq:logging policy}
\end{align}
The policy $\pi_0$ is random with probability $\varepsilon$ and takes near-optimal actions otherwise. We set $\varepsilon = 0.1$. This mimics a real-world setting where the deployed policy is already of a high quality.

\textbf{Evaluation.} We evaluate all methods by \emph{recovered hypervolume}, which is the hypervolume of their solutions over its estimated maximum. Since all experiments are simulations, we know $V$ and approximate the maximum hypervolume as $\vol(\tilde{S}, f)$, where $\tilde{S}$ are $10\,000$ random policies, chosen as in \random. The recovered hypervolume is averaged over $20$ runs and we also report the standard error of its estimate. We log up to $n = 30\,000$ points with noise $\sigma = 1$ in each run.

\subsection{Results}
\label{sec:experiments results}

\begin{figure*}[t]
  \centering
  \includegraphics[width=16cm]{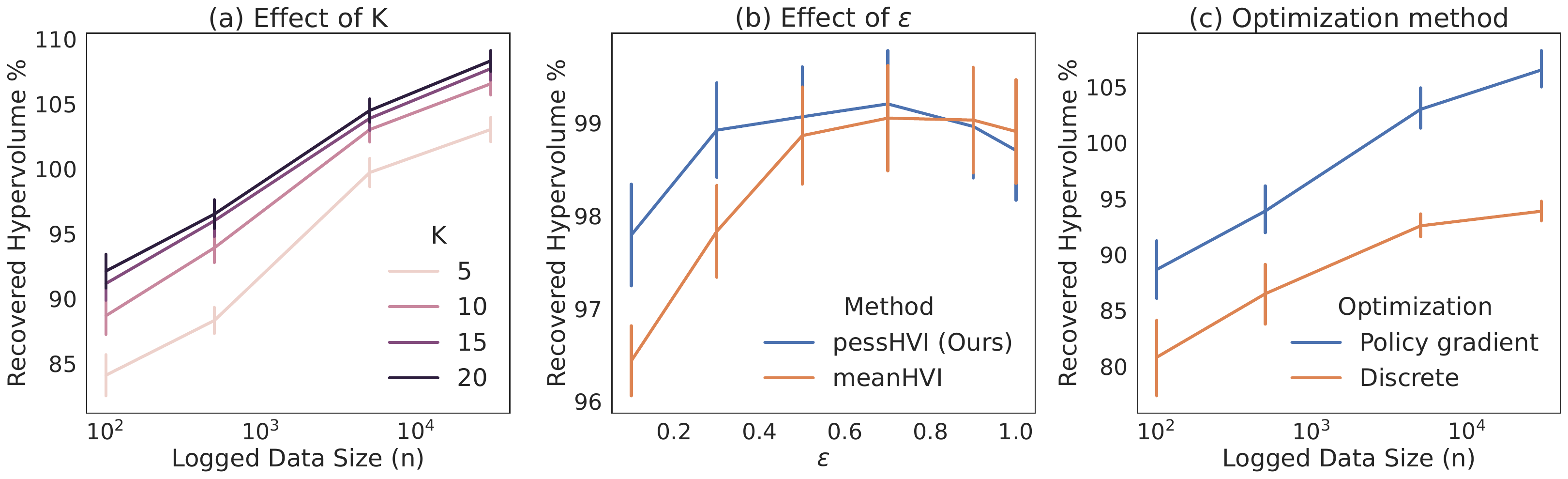}
  \vspace{-0.05in}
  \caption{(a) Recovered hypervolume by \pesshvi as a function of $K$ and logged dataset size $n$. (b) The benefit of pessimism diminishes as the logging policy becomes more uniform, $\varepsilon \to 1$. (c) Comparison of the recovered hypervolume by policy gradient and discrete optimization in \pesshvi with $K = 10$.}
  \label{fig:ablation study}
\end{figure*}

In \cref{fig:n,fig:K}, we report the performance of all methods on selected ZDT, DTZL, and WFG benchmarks with $m = 2$ objectives and $d = 6$ features. In \cref{fig:n}, we fix the number of optimized policies at $K = 10$ and vary the logged dataset size $n$. In \cref{fig:K}, we fix the logged dataset size at $n = 500$ and vary the number of optimized policies $K$. We observe two major trends. First, all methods generally improve as we increase $n$ or $K$. This is expected since larger sample sizes $n$ yield better estimates of $V$ and larger $K$ make hypervolume optimization easier. Second, \pesshvi consistently outperforms \meanhvi and all baselines. This shows that pessimistic estimators are more robust to optimization from logged data, as suggested by our analysis in \cref{sec:analysis}.

We present additional results on $5$ ZDT, $7$ DTLZ, and $9$ WFG problems in \cref{sec:additional experiments}; and observe similar trends to \cref{fig:n,fig:K}. We also include $7$ DTLZ problems with more objectives and features.

In \cref{fig:ablation study}, we conduct an ablation study of recovered hypervolume by \pesshvi on DTLZ2 problem. In \cref{fig:ablation study}a, we vary the logged dataset size $n$ and the number of optimized policies $K$. The recovered hypervolume improves in both. This is expected, since larger sample sizes $n$ yield better estimates of $V$ and larger $K$ make hypervolume optimization easier. In \cref{fig:ablation study}b, we vary $\varepsilon$ in the logging policy $\pi_0$, and set $n = 5\,000$ and $K = 10$. We observe that the recovered hypervolume by \pesshvi approaches that of \meanhvi as $\pi_0$ becomes more uniform, $\varepsilon \to 1$. This confirms our theory in \cref{sec:discussion}, showing that the benefit of pessimism diminishes when all confidence intervals have comparable widths. Finally, in \cref{fig:ablation study}c, we compare policy-gradient optimization (\cref{sec:policy gradient}) to discrete optimization over $1\,000$ random policies (\cref{sec:discrete optimization}). The discrete optimization yields subpar results, likely due to insufficient discretization. This is why we conduct continuous optimization by policy gradients.

\section{Related Work}
\label{sec:related work}

Our paper is at the intersection of several fields and we review prior works in detail in \cref{sec:additional related work}. Here we discuss only some. Our method is an instance of \emph{a-posteriori} methods, which cover the Pareto front by a diverse set of points. Notable approaches include random scalarization in ParEGO \citep{knowles2006parego} and evolutionary methods, such as MOEA/D and NSGA-II \citep{zhang2007moea,deb02fast}. The \emph{hypervolume indicator} has become the metric of choice in several recent works that provide guarantees \citep{auer16pareto,zhang20random}.

MOO has been studied extensively in the online setting, where the learning agent interactively explores the Pareto front \citep{drugan13designing}. Both upper confidence bound and posterior sampling methods were proposed \citep{auer16pareto,yahyaa15thompson}, some of which are designed for Gaussian processes \citep{zuluaga2013active,paria19flexible}. Multi-objective reinforcement learning (RL) is a natural generalization of a single-step optimization and an active research area \citep{hayes2022practical}.

MOO is relatively understudied in the off-policy optimization setting. \citet{wang22imo3} formalized this problem as optimizing a scalarized objective, where the scalarization is learned by interacting with a policy designer. In contrast, our method is a-posteriori, produces a set of diverse policies without any human input, and incorporates pessimism. Two recent offline RL papers also assumed a known scalarization \citep{satija21multiobjective,wu21offline}. Both consider a form of pessimism and apply it to finite-state models. These methods are a-priori because the scalarization is assumed to be known. Our method is a-posteriori and does not assume that the context set is finite. Finally, \citet{zhu23scaling} used hypervolume to obtain expert trajectories in offline multi-objective RL. This work is empirical and does not use pessimism. In comparison, we show the value of pessimistic hypervolume optimization in both theory and experiments.

\section{Conclusions}
\label{sec:conclusions}

We study offline optimization of multi-objective policies. We propose a practical a-posteriori approach to this problem, which maximizes a pessimistic hypervolume estimate for a diverse set of policies. The maximization is done by applying policy gradient jointly to all policies. We showcase the benefit of pessimism both theoretically and empirically.

This is one of the first works on offline optimization of multi-objective policies (\cref{sec:related work}). We analyze the benefit of pessimism generally (\cref{sec:analysis}), beyond our IPS estimators (\cref{sec:off-policy estimation}) and methods for optimizing them (\cref{sec:hypervolume optimization}). Consequently, our results are of a general interest, and could lay ground for analyzing other notions of diversity in a-posteriori MOO. One shortcoming of our approach is that each objective is modeled separately. Therefore, we do not take advantage of correlations among the objectives, which could improve statistical efficiency.

\bibliographystyle{abbrvnat}
\bibliography{References,Brano}



\clearpage
\onecolumn

\begin{toappendix}
\section{Additional Experiments}
\label{sec:additional experiments}

We conduct additional experiments on ZDT \citep{zitzler2000comparison}, DTLZ \citep{deb2005scalable}, and WFG \citep{huband05scalable} problems. The setting is the same as in \cref{sec:experiments results}.

\subsection{ZDT Problems}
We experiment with $5$ ZDT problems out of $6$, with $m = 2$ objectives and $d = 6$ features. ZDT5 is excluded since it is a discrete optimization problem. The remaining $5$ problems are continuous. Our results are reported in \cref{fig:zdt}. We observe that \pesshvi consistently improves upon all baselines when $n \geq 500$.
\vspace{-0.1in}

\begin{figure*}[h!]
  \centering
  \includegraphics[width=0.94\textwidth]{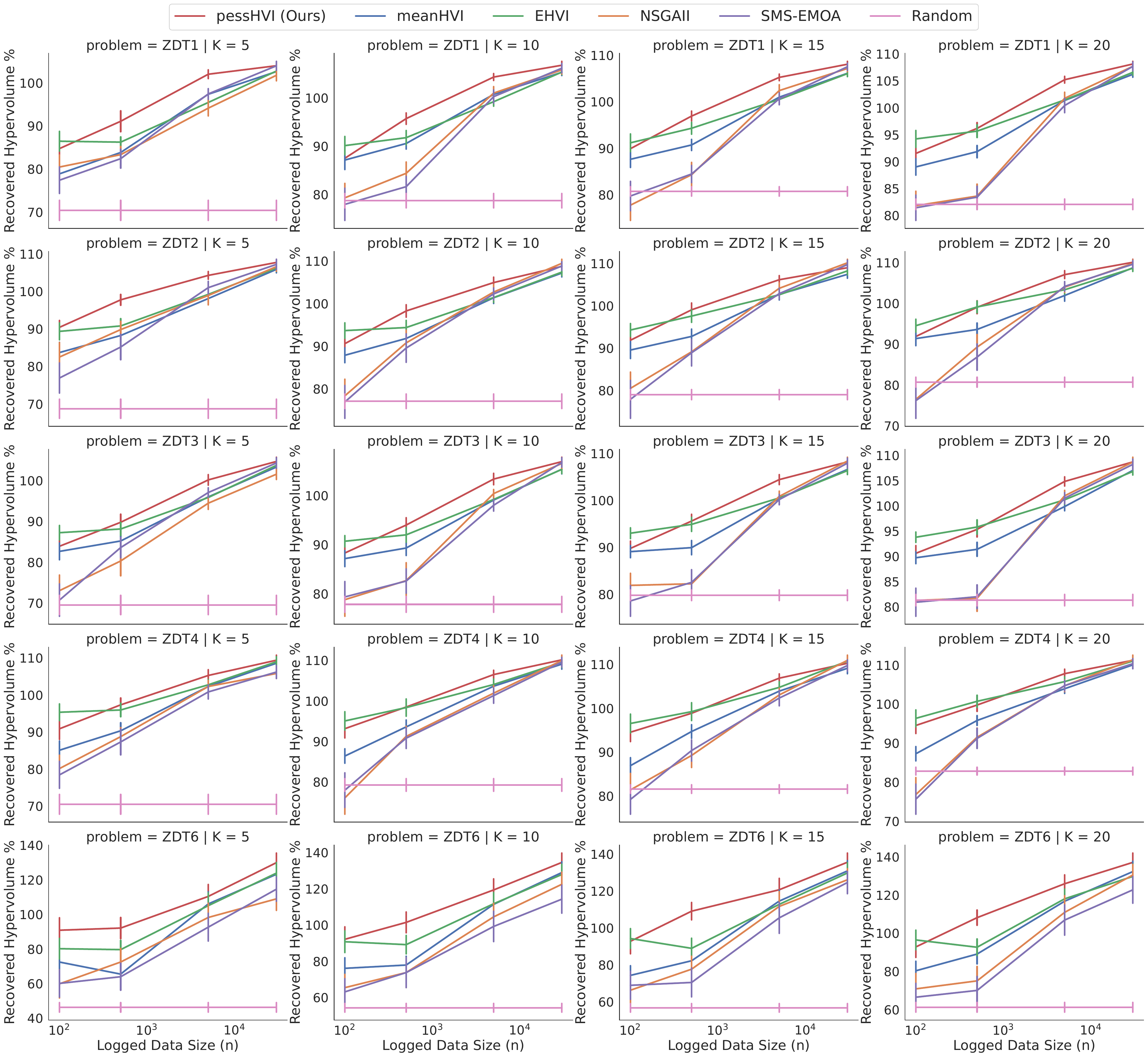}
  \caption{Evaluation of \pesshvi and all baselines on $5$ ZDT problems, for different values of $K$ and $n$.}
  \label{fig:zdt}
\end{figure*}

\subsection{DTLZ Problems}
\label{sec:dtlz problems}

We experiment with $7$ DTLZ problems out of $9$, with $m = 2$ objectives and $d = 6$ features. We exclude DTLZ8 and DTLZ9 because these problems are constrained. The remaining $7$ problems are unconstrained. Our results are reported in \cref{fig:dtlz}. We observe that \pesshvi consistently improves upon all baselines when $n \geq 500$. The only exception is DTLZ6, where many methods perform well.
\vspace{-0.1in}

\begin{figure*}[h!]
  \centering
  \includegraphics[width=0.94\textwidth]{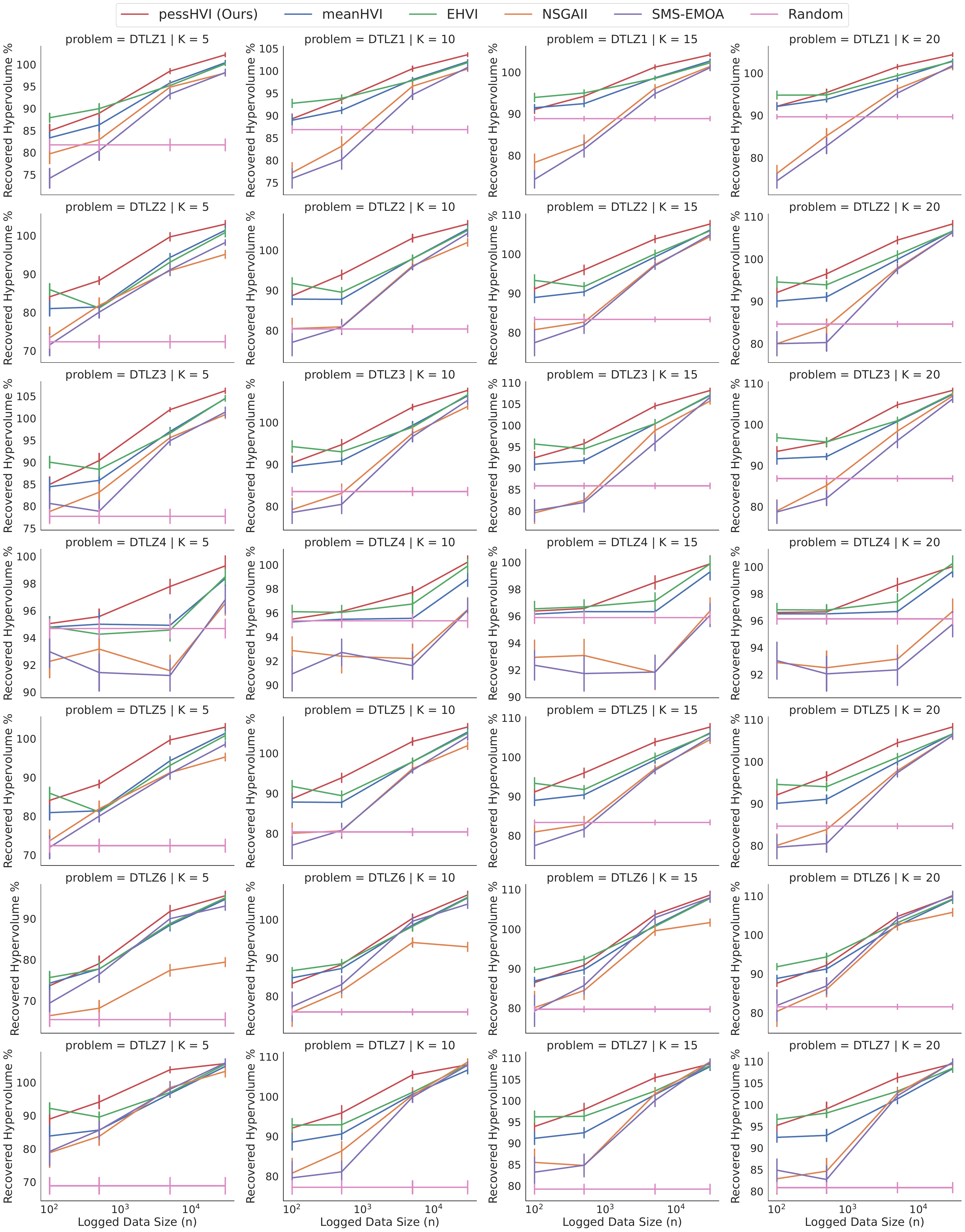}
  \caption{Evaluation of \pesshvi and all baselines on $7$ DTLZ problems, for different values of $K$ and $n$.}
  \label{fig:dtlz}
\end{figure*}

\newpage

\subsection{DTLZ Problems with $4$ Objectives}
Similarly to \cref{sec:dtlz problems}, we experiment with $7$ DTLZ problems, with $m = 4$ objectives and $d = 10$ features. The hypervolume is computed as described in \cref{sec:random hypervolume scalarization}. Our results are reported in \cref{fig:dtlz4}. We observe that \pesshvi consistently improves upon all baselines in $5$ problems. In DTLZ4 and DTLZ6, \pesshvi performs comparably to \meanhvi and \ehvi.

\vspace{-0.05in}
\begin{figure*}[h!]
  \centering
  \includegraphics[width=1\textwidth]{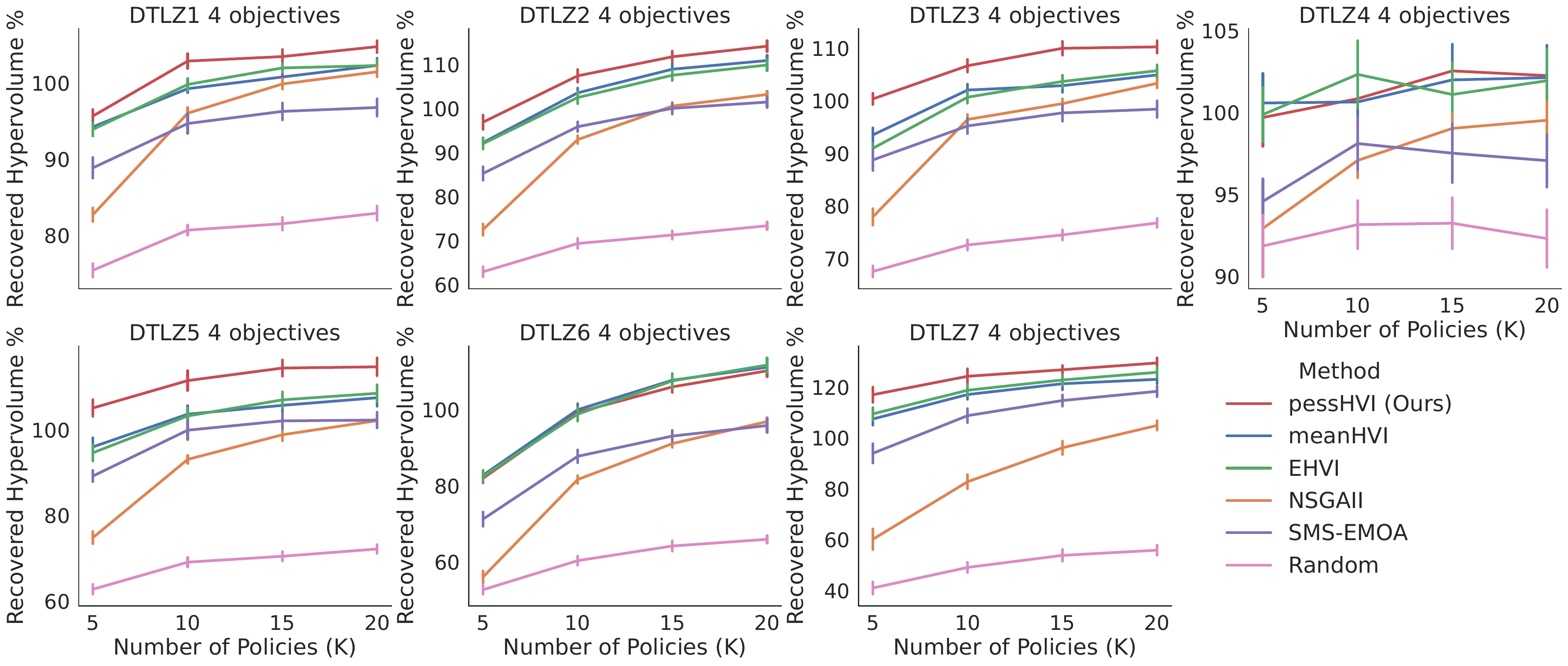}
  \caption{Evaluation of \pesshvi and all baselines on $7$ DTLZ problems with $4$ objectives. We set $n = 5000$ and vary $K$.}
  \label{fig:dtlz4}
\end{figure*}

\vspace{-0.1in}

\subsection{WFG Problems}

We experiment with $9$ WFG problems, with $m = 2$ objectives and $d = 6$ features. Our results are reported in \cref{fig:wfg}. We observe that \pesshvi consistently improves upon all baselines when $n \geq 500$. The only exception is WFG2, where many methods perform well.

\begin{figure*}[h!]
  \centering
  \includegraphics[width=0.81\textwidth]{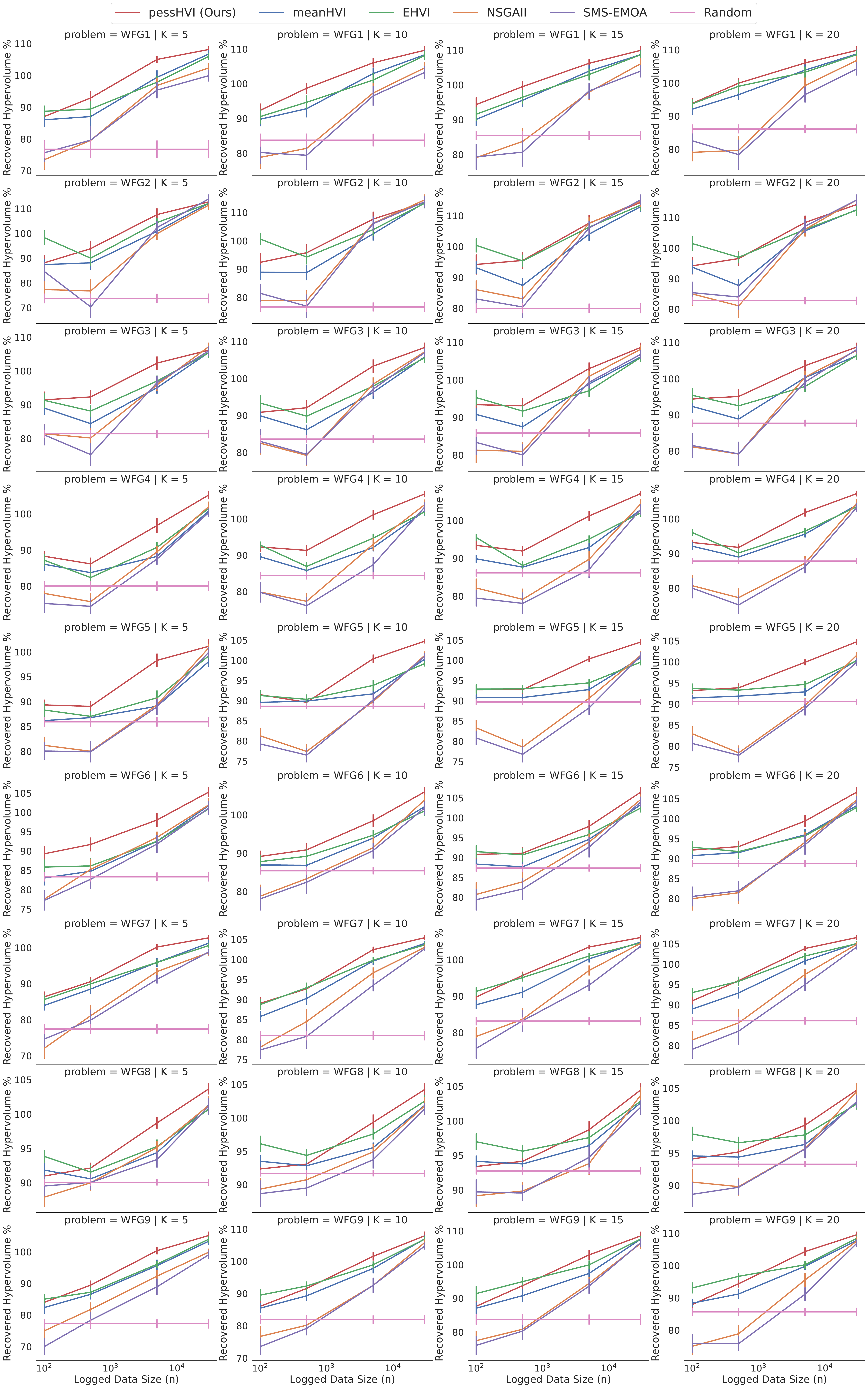}
  \vspace{-0.1in}
  \caption{Evaluation of \pesshvi and all baselines on $9$ WFG problems, for different values of $K$ and $n$.}
  \label{fig:wfg}
\end{figure*}

\section{Other Estimators}
\label{sec:other estimators}

The \emph{direct method (DM)} \citep{dudik14doubly} is a popular approach to off-policy evaluation. Using the DM, the value of policy $\pi$ in objective $i$ can be computed as
\begin{align*}
  \hat{V}^\textsc{dm}_i(\pi)
  = \frac{1}{n} \sum_{t = 1}^n \sum_{a \in \cA} \pi(a \mid x_t) \hat{r}_i(x_t, a)\,,
\end{align*}
where $\hat{r}_i(x, a)$ is the empirical mean estimate of $r_i(x, a)$.

The \emph{doubly-robust method (DR)} \citep{robins94estimation,dudik14doubly} combines the DM and IPS as
\begin{align*}
  \hat{V}^\textsc{dr}_i(\pi)
  = \frac{1}{n} \sum_{t = 1}^n
  \frac{\pi(A_t \mid x_t)}{\pi_0(A_t \mid x_t)} (Y_{t, i} - \hat{r}_i(x_t, A_t)) + \hat{V}^\textsc{dm}(\hat{\pi})\,.
\end{align*}
It is popular because it combines the advantages of the DM and IPS: it is unbiased when the DM estimator is unbiased or the propensities in the IPS estimator are correctly specified.

A \emph{self-normalized IPS (SNIPS)} estimator,
\begin{align*}
  \hat{V}^\textsc{snips}_i(\pi)
  = \frac{1}{\sum_{t = 1}^n \frac{\pi(A_t \mid x_t)}{\pi_0(A_t \mid x_t)}}
  \sum_{t = 1}^n \frac{\pi(A_t \mid x_t)}{\pi_0(A_t \mid x_t)} Y_{t, i}\,,
\end{align*}
is another popular approach. Unlike IPS, it is bounded but biased. However, if the logging policy is supported for all actions in each context, it is consistent \citep{swaminathan15selfnormalized}.

\section{Gradient-Based Methods for Diverse Points}
\label{sec:diverse gradient}

Gradient-based methods for finding a diverse set of points have been previously studied in the multi-task and multi-objective optimization literature. They fall into two categories: $K$ points are optimized directly or their hypervolume is. An early approach in the former direction is algorithm MGDA of \citet{desideri2012multiple}, which uses KKT conditions to compute the direction in which all objectives increase. Various extensions of MGDA have been proposed \citep{sener2018multi,zhou2022convergence,liu2021profiling}. These works do not directly target diversity. Gradient ascent on the hypervolume, as in our work, has been studied. One of the first works on this topic is \citet{wang2017hypervolume}, who used hypervolume gradient derivations of \citet{emmerich2014time}. The main challenge for this approach is that the hypervolume indicator is locally constant if any point is dominated. To avoid this, various modifications to steer dominated points to the boundary have been proposed \citep{deist2020multi,deist2021multi}. We believe that these methods could improve our gradient ascent optimization in \cref{sec:policy gradient}.

\section{Hypervolume Computation}
\label{sec:hypervolume}

We review several existing hypervolume estimators. An exact formula based on the exclusion-inclusion principle is presented in \cref{sec:inclusion-exclusion estimator}. Unfortunately, it is computationally intractable when $S$ is large. In \cref{sec:two objectives}, we discuss an exact formula for two objectives that has $O(\abs{S} \log \abs{S})$ computation time. Finally, we present an approximation with $O(\abs{S})$ computation time in \cref{sec:random hypervolume scalarization}. All formulas are stated for any multi-objective function $f: \Pi \to \realset^m$.

\subsection{Inclusion-Exclusion Estimator}
\label{sec:inclusion-exclusion estimator}

The key insight in the inclusion-exclusion estimator \citep{daulton2020differentiable} is that the area of the union of two rectangles is the sum of their areas minus the area of their intersection, which is also a rectangle. In general, for hyperrectangles $\bigtimes_{i = 1}^m [a, f_i(\pi)]$, the hypervolume can be computed as follows. Let $2^S$ be the power set of $S \subseteq \Pi$. Then
\begin{align}
  \vol(S, f)
  = \sum_{C \in 2^S \setminus \emptyset} (2 \, (\abs{C} \, \mathrm{mod} \, 2) - 1)
  \prod_{i = 1}^m \left(\min_{\pi \in C} f_i(\pi) - a\right)\,,
  \label{eq:inclusion-exclusion hypervolume}
\end{align}
where $a \in \realset$ represents a reference point for the beginning of the coordinate system. The computation of \eqref{eq:inclusion-exclusion hypervolume} takes $O(2^{\abs{S}})$ time and therefore is inefficient even for relatively small $S$.

\subsection{Two Objectives}
\label{sec:two objectives}

For $m = 2$ objectives, algorithms with $O(\abs{S} \log \abs{S})$ computation time exist. More specifically, let $S = \set{\pi_k}_{k = 1}^K$ and suppose that $f_1(\pi_1) \leq \dots \leq f_1(\pi_K)$ holds, which can be done in $O(\abs{S} \log \abs{S})$ time by sorting $\pi_k$ according to the first objective \citep{preparata2012computational}. Then $f$ can be treated as a single-objective function, where $f_1(\pi_k)$ is its input and $f_2(\pi_k)$ is its output, and integrated along the first objective as
\begin{align*}
  \vol(S, f)
  = (f_1(\pi_1) - a) (\max_{k \in [K]} f_2(\pi_k) - a) +
  \sum_{k = 1}^{K - 1} (f_1(\pi_{k + 1}) - f_1(\pi_k))
  (\max_{\ell \in [K] \setminus [k]} f_2(\pi_\ell) - a)\,.
\end{align*}

\subsection{Random Hypervolume Scalarization}
\label{sec:random hypervolume scalarization}

Scalarization is a mapping $s_\lambda(f(\pi)): \realset^m \to \realset$ for a given $\lambda \in \realset^m$. The key idea in all scalarization methods is to reduce multiple objectives into a scalar and then optimize it. The most common scalarization techniques are linear $s_\lambda(f(\pi)) = \sum_{i = 1}^m \lambda_i f_i(\pi)$ and Chebyshev $s_\lambda(f(\pi)) = \min_{i \in [m]} \lambda_i (f_i(\pi) - a_i)$, where $a \in \realset^m$ is a reference point. Random hypervolume scalarization approximates the hypervolume indicator with random scalarizations chosen from an appropriate distribution. Specifically, \citet{zhang20random} showed that the hypervolume $\vol(S, f)$ can be rewritten as
\begin{align*}
  \vol(S, f)
  \propto \Erv{\lambda \sim B_m}{\max_{\pi \in S} s_{\lambda}(f(\pi) - a)}\,,
\end{align*}
where $s_{\lambda}(y) = \min_{i \in [m]} \max \set{0, y_i / \lambda_i}^m$, $B_m$ is a unit sphere in $\realset^m$, and vector $\lambda$ is drawn uniformly from $B_m$. The expectation is approximated by sampling $\lambda$.

\section{Additional Related Work}
\label{sec:additional related work}

In general multi-objective optimization, the decision maker must choose a candidate $x$ from a set of potential candidates $\cX$. For each $x \in \cX$, there are $m$ objective values $f(x) = (f_i(x))_{i = 1}^m$, where $f_i: \cX \to \realset$. Because the objectives can be traded off in many ways, many algorithms for MOO exist \citep{emmerich18tutorial}.

In the \emph{a-priori setting} \citep{branke2008multiobjective}, the utility of a decision maker is known in advance and used to find the optimal candidate. It is common to represent the utility function as belonging to a family of \emph{scalarizations} of the objectives, where the objectives are weighted separately and then combined. Arguably the most popular approach is linear scalarization $s_\lambda(f(x)) = \sum_{i = 1}^m \lambda_i f_i(x)$, where $\lambda \in \realset^m$ is a weight vector. In many real-world problems, $\lambda$ is unknown in advance. In such cases, it is natural to present potential candidates to the decision maker that approximate the Pareto front well. This is known as the \emph{a-posteriori setting} \citep{branke2008multiobjective} and many algorithms exist for it. One popular approach is to cover the Pareto front using random scalarization \citep{zhang20random}. This was done in ParEGO \citep{knowles2006parego} and an evolutionary algorithm MOEAD \citep{zhang2007moea}. Other evolutionary algorithms, such as NSGA-II \citep{deb02fast}, iteratively refine a population of candidates based on various fitness metrics. Unlike our approach, none of these methods provide guarantees on the quality of the approximation and additionally do not not handle uncertainty in objectives.

Regardless of the MOO method, the quality of the resulting solution needs to be measured. Intuitively, a good approximation contains a set of points that are close to the Pareto front and sufficiently diverse. Metrics that capture these two qualities are called performance indicators \citep{zitzler2008quality, audet2021performance}. Popular indicators are the Hausdorff distance from the approximation to the Pareto front, $R_2$, and hypervolume \citep{zitzler2000comparison}. The last has been increasingly popular and considered in several recent works \citep{zhang20random, auer16pareto}. As discussed in \cref{sec:hypervolume}, hypervolume can be challenging to compute.

In the \textit{online setting}, the decision maker interactively explores the Pareto front. \cite{drugan13designing} was the first work to apply bandits to MOO. They proposed a \ucb algorithm with a scalarized objective and also a Pareto \ucb algorithm. \cite{auer16pareto} formulated the problem of the Pareto front identification as best-arm identification where at each round, a point $x$ is chosen and a noisy observation of the objective $f(x)$ is observed. Thompson sampling in MOO was studied in \cite{yahyaa15thompson}.  A popular paradigm in MOO is to assume that the objective functions are drawn from a Gaussian process (GP) and several works have taken that approach. \cite{zuluaga2013active} is an early work with theoretical guarantees that has a similar observation model to \cite{auer16pareto}. Two recent works that applied GP bandits to MOO are \cite{paria19flexible} and \cite{zhang20random}. \cite{paria19flexible} minimizes the regret with respect to a known distribution of scalarization vectors. \cite{zhang20random} showed that this algorithm generates a set of points that maximize random hypervolume scalarization. All above works are in the online setting, where the learning agent can interactively probe the environment to learn about objective functions. Our setting is offline.

Arguably the two closest works are \citet{roijers17interactive} and \citet{wang22imo3}. \citet{roijers17interactive} treated online MOO as a two-stage problem, where the objective functions are estimated using initial interactions with the environment and the scalarization vector is then estimated by interacting with the designer. This approach was further refined by \citet{wang22imo3}, who used state-of-the-art off-policy estimation techniques to estimate the objectives and analyzed their approach. The key difference in our work is that we do not put any interaction burden on the policy designer, and simply present them a diverse set of policies to choose from.

\end{toappendix}

\end{document}